\documentclass[conference]{IEEEtran}
\IEEEoverridecommandlockouts
\usepackage{cite}
\usepackage{amsmath,amssymb,amsfonts}
\usepackage{graphicx}
\usepackage{bm}
\usepackage{hyperref}
\usepackage{graphicx}
\usepackage{amsmath}
\usepackage{cleveref}
\usepackage{amssymb}
\usepackage{textcomp}
\usepackage{algorithm}
\usepackage[noend]{algpseudocode}
\usepackage{stackengine}
\usepackage{hyperref}
\usepackage{xcolor}
\usepackage{graphicx}
\usepackage{framed}
\usepackage{array}
\usepackage{subcaption}
\usepackage{booktabs}
\usepackage{comment}
\usepackage{siunitx}
\usepackage[font=small,labelfont=bf]{caption}
\def\BibTeX{{\rm B\kern-.05em{\sc i\kern-.025em b}\kern-.08em
    T\kern-.1667em\lower.7ex\hbox{E}\kern-.125emX}}
\newcolumntype{L}[1]{>{\raggedright\arraybackslash}p{#1}}
\algnewcommand{\algorithmicor}{\textbf{ or }}
\algnewcommand{\OR}{\algorithmicor}
    
\usepackage{amsthm}
\usepackage{tikz}
\usetikzlibrary{external}
\crefname{equation}{}{}
\Crefname{equation}{}{}
\crefname{figure}{Fig.}{Fig.}

\newcommand{\ie}{\textit{i}.\textit{e}., }
\newcommand{\eg}{\textit{e}.\textit{g}., }

\begin{document}
\newtheorem{theorem}{Theorem}
\newtheorem{remark}{Remark}
\newtheorem{lemma}{Lemma}
\newtheorem{proposition}{Proposition}
\newtheorem{definition}{Definition}

\title{Provably Safe Tolerance Estimation for Robot Arms via Sum-of-Squares Programming
}

\author{Weiye Zhao, Suqin He, and Changliu Liu
\thanks{* This work is in part supported by ARM Institute and Amazon Research Award.}
\thanks{W. Zhao and C. Liu are with Carnegie Mellon University. Emails: \texttt{weiyezha, cliu6@andrew.cmu.edu}. S. He is with Tsinghua University (work done while at CMU). Email: \tt hesq16@mails.tsinghua.edu.cn.}}

\maketitle

\begin{abstract}
Tolerance estimation problems are prevailing in engineering applications. For example, in modern robotics,  it remains challenging to efficiently estimate joint tolerance, \ie the maximal allowable deviation from a reference robot state such that safety constraints are still satisfied. This paper presented an efficient algorithm to estimate the joint tolerance using sum-of-squares programming. It is theoretically proved that the algorithm provides a tight lower bound of the joint tolerance. Extensive numerical studies demonstrate that the proposed method is computationally efficient and near optimal. The algorithm is implemented in the JTE toolbox and is available at \url{https://github.com/intelligent-control-lab/Sum-of-Square-Safety-Optimization}.

\end{abstract}

\section{INTRODUCTION}

Tolerance estimation problems are prevailing in engineering applications. The problem is to find a mapping from a constraint $f(\bm{x})\geq 0$, a feasible point for that constraint $\bm{x}^r$, and a distance metric (e.g., $l_p$), to the maximum allowable deviation $\lambda$ to that feasible point in the distance metric such that the constraint is still satisfied. Mathematically, the problem $(f,\bm{x}^r,p)\mapsto \lambda$ can be written as the following optimization:
\begin{align}\label{eq: fundemantal problem}
    \max \lambda \text{ s.t. }\forall \|\bm{x} - \bm{x}^r\|_{p} \leq \lambda, f(\bm{x}) \geq 0. 
\end{align}
The problem arises when computing the convex feasible set for a collision avoidance constraint \cite{liu2018convex}; when computing the error tolerance for a given trajectory \cite{WU200754}; when verifying the adversarial bound for a given constraint \cite{OPT-035}; and when verifying regions of attraction for nonlinear systems~\cite{tedrake2010lqr}. We call the constraint in \eqref{eq: fundemantal problem} a \textit{local positiveness constraint}, which requires $f$ to be positive in a bounded neighbourhood of $\bm{x}^r$.

The solution strategies for \eqref{eq: fundemantal problem} differ significantly for different types of constraints $f$. This paper focuses on finding tolerance bounds in the joint space of robot arms with respect to  collision avoidance constraints in the Cartesian space. We call this problem as a \textit{joint tolerance estimation} problem. Hence the function $f$ contains the forward kinematics of the robot arm and the distance computation in the Cartesian space, which is continuous but highly nonlinear and nonconvex. The goal is to find a non-conservative lower bound of \eqref{eq: fundemantal problem}, which can then be used to guide the design of subsequent controllers to safely stabilize the robot arm within the bound. 

While there is limited work on solving \eqref{eq: fundemantal problem} that specifically leverages the features of $f$ for robot arms, there are many generic methods that solve \eqref{eq: fundemantal problem} in general. The two major types of solutions are through adversarial optimization and reachability analysis. Adversarial optimization \cite{Adversarial} solves an equivalent form of \eqref{eq: fundemantal problem}:
\begin{align}\label{eq: fundemantal equivalent form}
    \inf \|\bm{x} - \bm{x}^r\|_{p} \text{ s.t. } f(\bm{x}) < 0.
\end{align}
which finds the smallest possible deviation such that the constraint is about to be violated. Numerically, this equivalent form \eqref{eq: fundemantal equivalent form} is easy to solve since it can be directly fed to existing optimization solvers, while \eqref{eq: fundemantal problem} cannot since it essentially contains infinitely many constraints (i.e., every $\bm{x}$ corresponds to a constraint and there are infinitely many $\bm{x}$ that satisfies $\|\bm{x} - \bm{x}^r\|_{p}$). However, due to local optimality induced by the complexity of the constraint $f$, it is challenging to find a lower bound for \eqref{eq: fundemantal equivalent form}~\cite{bazaraa2013nonlinear}. On the other hand, reachability based methods \cite{Reachability_10.1007/978-3-540-24743-2_10} solves the problem by proposing different $\lambda$, computing the reachable set $\mathcal{R}(\lambda):=\{f(\bm{x})\mid \|\bm{x} - \bm{x}^r\|_{p}\leq \lambda\}$, and then performing binary search to find the maximal $\lambda$ such that all points in $\mathcal{R}(\lambda)$ satisfies $f(\bm{x})\geq 0$. However, exact reachable sets are expensive to compute, while under-approximated reachable sets will not provide a lower bound of \eqref{eq: fundemantal problem} and over-approximated reachable sets may render the problem overly conservative.   

This paper proposes to directly deal with the infinite cardinality of constraints by leveraging techniques from sum-of-squares programming (SOSP)~\cite{prajna2002introducing}. Although SOSP is able to deal with infinitely many constraints, the standard SOSP formulation only deals with \textit{global positiveness constraints} instead of \textit{local positiveness constraints}~\cite{papachristodoulou2013sostools}. A \textit{global positiveness constraint} can be written as $\forall\bm{x},f(\bm{x})\geq0$, which is more restrictive than a \textit{local positiveness constraint}. To derive a non conservative solution of \eqref{eq: fundemantal problem}, we then resort to the \textit{Positivstellensatz} condition, \ie the theory behind SOSP, to turn the problem with \textit{local positiveness constraints} into an ordinary nonlinear program. The nonlinear program can then be easily solved by off-the-shelf nonlinear programming solvers. It is formally proved that the nonlinear program provides a lower bound of \eqref{eq: fundemantal problem}.   

The contributions of this paper include:
\begin{itemize}
    \item We formulate the joint tolerance estimation problem with \textit{local positiveness constraints} into a sum-of-squares program and then an ordinary nonlinear program by using the \textit{Positivstellensatz} condition.
    \item We develop an efficient algorithm for joint tolerance estimation which is able to find tight lower bounds of \eqref{eq: fundemantal problem} when $p=\infty$. The algorithm is implemented in the JTE toolbox and is available at \url{https://github.com/intelligent-control-lab/Sum-of-Square-Safety-Optimization}. 
\end{itemize}

\section{Problem Formulation}
For an $n$ degree-of-freedom (DOF) robot manipulator, define $x_i$ as the angle of joint $i$. The state of the robot is defined as $\bm{x} = [x_1; x_2; \ldots; x_n]$. Define a distance function $f(\bm{x})$ to compute the closest distance between the robot at configuration $\bm{x}$ and the environmental obstacles $\mathcal{O}\subset\mathbb{R}^3$ in Cartesian space, \ie $f(\bm{x}) = \min_{\bm{p}\in \mathcal{C}(\bm{x}),\bm{o}\in \mathcal{O}} d(\bm{p}, \bm{o})$ where $d$ measures the distance between two Cartesian points. The notation $\mathcal{C}(\bm{x})$ refers to the area occupied by the robot in the Cartesian space at configuration $\bm{x}$, which needs to be computed using forward kinematics~\cite{merlet2004solving}. Forward kinematics involve  trigonometric terms, \ie $\sin(\cdot), \cos(\cdot)$. 
Hence, $f(\bm{x})$ is highly nonlinear and is not polynomial. 

Suppose these is a feasible reference $\bm{x}^r$, we want to find the maximum tolerance bound $\lambda$ such that if the deviation from the actual robot state $\bm{x}$ to $\bm{x}^r$ does not exceed $\lambda$, the robot is still safe. The problem can be formulated into \eqref{eq: fundemantal problem}. This paper focuses on $p=\infty$ case. Hence the joint tolerance estimation problem can be posed as
\begin{equation}
    \max \lambda, \text{ s.t. } \forall \|\bm{x} - \bm{x}^r\|_{\infty} \leq \lambda, f(\bm{x}) \geq 0 .
    \label{probori}
\end{equation}

Here we provide an example on a 2-link planar robot arm shown in \cref{fig:1}. The problem is to find the maximum allowable deviation from the reference pose such that the robot does not collide with the wall at $x_{wall}$. In this example, since the first link of the robot cannot reach the wall, we only need to ensure that the end-effector of the robot does not hit the wall. Suppose the length of the two links are both $1$ and $x_i$ is the $i$-th joint angle for the robot in world frame, then the collision avoidance constraint can be written as $f(\bm{x}) = x_{wall} - \cos{x_1} - \cos{x_2}$. The constraint $f(\bm{x})$ and the reference $\bm{x}^r$ are also illustrated in the configuration (state) space. The tolerance is illustrated as the green square in the configuration space.

\begin{figure}
    \centering
    \includegraphics[width=0.9\columnwidth]{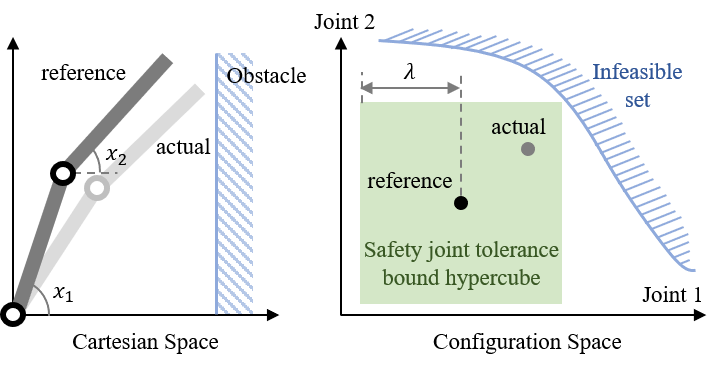}
    \caption{Illustration of joint tolerance estimation for a 2-link robot arm. The actual position of the robot arm may deviate from its reference. We need to compute the maximum allowable deviation such that the robot does not collide with the obstacle.}
    \label{fig:1}
\end{figure}

As mentioned earlier, \eqref{probori} contains infinitely many constraints since every feasible $\bm{x}$ poses an inequality constraint, hence cannot be directly solved by off-the-shelf nonlinear programming solvers. To deal with this issue, we leverage tools for SOSP (to be reviewed in the following section) to repose the problem as a nonlinear program. 
Nonetheless, using these tools requires $f$ to be a polynomial in $\bm{x}$. We can turn $f$ into a polynomial through local Taylor expansion. To control the quality of the Taylor approximation, it is desired to constrain the free variable (\eg $\bm{x}$) in a fixed range instead of a varying range defined by $\lambda$. Hence we define an auxiliary variable $\bm{y} = [y_1,y_2, \ldots, y_n] \in \mathbb{R}^n$ and rewrite \cref{probori} as following: 
\begin{equation}
      \begin{split}
          \max~&  \lambda\\
          \text{s.t. }
                & \forall y_i^2 \leq 1, i = 1,2,\ldots,n,f(\bm{x}^r + \bm{y}\lambda) \geq 0.
      \end{split}
      \label{probtrans}
\end{equation}
In the following discussion, we review the tools from SOSP and show how to use these tools to turn \eqref{probtrans} into an ordinary nonlinear program.

\section{Sum of Squares Related Work}\label{sec:sos}
Optimization problems with \textit{global positiveness constraints} are widely studied and have been successfully applied to many problems, such as synthesizing Lyapunov functions~\cite{tan2004searching}. These problems can be solved by leveraging SOSP~\cite{papachristodoulou2013sostools}. Mathematically, a polynomial $p(\bm{x})$ of degree $2d$ is SOS (denoted as $p(\bm{x})\in\text{SOS}$) if and only if there exists polynomials $e_1(\bm{x}),\ldots,e_k(\bm{x})$ of degree $d$ such that $p(\bm{x}) = \sum_{i=1}^{k}e_k(\bm{x})^2$. In particular, SOSTOOLS~\cite{papachristodoulou2013sostools} implements SOSP to solve the problems with \textit{global positiveness constraints}. 
However, our problem has \textit{local positiveness constraints}, which cannot be directly written into the form that SOSTOOLS accepts. We need to dig into the theory behind SOSP.

To ensure \textit{global positiveness} of a condition, the easiest approach is to show that there does not exist any solution that the condition is violated. The set such that the condition is violated is called a refute set. For example, the refute set of $\forall\bm{x}, f(\bm{x})\geq 0$ is $\{\bm{x}\mid f(\bm{x})< 0\}$. Constructing the refute set and showing it is empty to ensure \textit{global positiveness} is the the core idea behind SOSP~\cite{parrilo2003semidefinite}. To show that the refute set is empty, we need to invoke the equivalence conditions in \textit{Positivstellensatz}~\cite{parrilo2003semidefinite}. 
 Before introducing \textit{Positivstellensatz}, let us first review the definitions of a key concept: ring-theoretic \textit{cone}~\cite{bochnak2013real}.
  \begin{definition}[Cone]
   Given a set $S \subseteq \mathbb{R}[x_1,\ldots,x_n]$, where $\mathbb{R}[x_1,\ldots,x_n]$ is a set of polynomials with $[x_1,\ldots,x_n]$ as variables. The general cone $\Gamma^*$ of $S$ is defined as a subset of $\mathbb{R}[x_1,\ldots,x_n]$ satisfying the following properties: 
  \begin{itemize}
      \item $a,b \in \Gamma^* \Rightarrow a+b \in \Gamma^*$
      \item $a,b \in \Gamma^* \Rightarrow a\cdot b \in \Gamma^*$
      \item $a\in\mathbb{R}[x_1,\ldots,x_n] \Rightarrow a^2 \in \Gamma^*$
  \end{itemize}
  \end{definition}
  
  \begin{definition}[Ring-theoretic cone]\label{prop}
  For a polynomial set $S = {\gamma_1,\ldots,\gamma_s} \subseteq \mathbb{R}[x_1,\ldots,x_n]$, the associated ring-theoretic \textit{cone} can be expressed as: 
  \begin{equation}
      \Gamma = \{ p_0 + p_1\gamma_1 + \ldots + p_s\gamma_s + p_{12}\gamma_1\gamma_2 +\ldots + p_{12\ldots s}\gamma_1\ldots\gamma_s\}
  \end{equation}
  where $\Gamma \subseteq \Gamma^*$ and $p_0, \ldots,p_{12\ldots s}\in\text{SOS}$.
  \end{definition}
  
  Based on the ring-theoretic \textit{cone}, \textit{Positivstellensatz} is specified in the following theorem.
  \begin{theorem}[\textit{Positivstellensatz}] \label{theo:posi}
   Let $(\gamma_j)_{j=1,\ldots,s}$, $(\psi_k)_{k=1,\ldots,t}$, $(\phi_l)_{l=1,\ldots,u}$ be finite families of polynomials in $\mathbb{R}[x_1,\ldots,x_n]$. Denote by $\Gamma$ the cone generated by $(\gamma_j)_{j=1,\ldots,s}$, $\Xi$ the multiplicative monoid~\cite{bochnak2013real} generated by $(\psi_k)_{k=1,\ldots.,t}$, and $\Phi$ the ideal~\cite{bochnak2013real} generated by $(\phi_l)_{l=1,\ldots,u}$. Then, the following properties are equivalent:
  \begin{enumerate}
      \item The following set is empty
          \begin{align}
          \begin{cases}
            x \in \mathbb{R}^n \Biggr|\begin{array}{ll}
                                    \gamma_j(\bm{x})\geq 0,j = 1,..,s\\
                                    \psi_k(\bm{x})\neq 0, k = 1,\ldots,t\\
                                    \phi_l(\bm{x})=0, l = 1,\ldots,u
                                    \end{array}
          \end{cases}.
          \end{align}
      \item There exist  $\gamma \in \Gamma, \psi\in \Xi, \phi \in \Phi$ such that 
      \begin{align}
          \gamma+\psi^2+\phi = 0,
          \label{eq:pos}
      \end{align}
  \end{enumerate}
  where $\psi = 1$ if $\psi_k(\bm{x})$ are empty, and $\phi = 0$ if $\phi_l$ is empty.
  \end{theorem}
  In summary, \textit{Positivstellensatz} shows that the refute set being empty is equivalent to a feasibility problem in \eqref{eq:pos}. The feasibility problem is easier to solve than the original problem with \textit{global positiveness constraints}. Nonetheless, \textit{Positivstellensatz} does not necessarily require the problem to be \textit{globally positive}. We can follow the same procedure as mentioned above to construct a refute set for the \textit{locally positive} problem \eqref{eq: fundemantal problem} and then use \textit{Positivstellensatz} to turn the problem into a feasibility problem similar to \eqref{eq:pos}, which can then be formed into an ordinary nonlinear program.

\section{Methodology}
In this section, we introduce the proposed optimization algorithm to efficiently solve problem \eqref{probtrans} leveraging \cref{theo:posi}. The overall procedure of the algorithm contains four main steps: 
\begin{itemize}
    \item To leverage techniques from SOSP, we construct lower bound polynomial functions to replace the constraints of \eqref{probtrans}. 
    \item To apply \textit{Positivstellensatz} to efficiently solve the lower bound problem, we construct its corresponding refute problem and try to show it is empty.
    \item To efficiently find solutions for refute problem, we construct the corresponding sufficient problem with reduced decision variable search space, and further transform it to a nonlinear programming problem that is easy to solve.
    \item Solve sufficient nonlinear programming problem using off-the-shelf nonlinear programming solver.
\end{itemize}

In the following subsections, we will introduce the details of three critical problems in prescribed four steps.

\subsection{Lower Bound Problem}
Since \Cref{theo:posi} only applies to polynomials in $\bm{x}$, we first need to bound $f(\bm{x}^r + \bm{y}\lambda)$ below using a polynomial. Specifically, we replace the trigonometric terms in $f(\bm{x}^r + \bm{y}\lambda)$ using small angle approximation, where $\sin(x) \approx x$ and $\cos(x) \approx 1 - \frac{x^2}{2}$, which results in a multi-variate polynomial $g(\bm{x}^r + \bm{y}\lambda)$. It has been shown that the relative error of small angle approximation is less than $ 1\%$ when $x < 0.244 $ \si{\radian}. In \cref{sec:result}, we empirically demonstrate that $g(\bm{x}^r + \bm{y}\lambda)$ is the lower bound of $f(\bm{x}^r + \bm{y}\lambda)$, and the rigorous proof is left for future work. For the 2-dimensional example shown in \cref{fig:1}, we have $f(\bm{x}^r + \bm{y}\lambda) = x_{wall} - \cos(x_1^r + y_1\lambda) - \cos(x_2^r + y_2\lambda)$. Then we can construct the corresponding $g(\bm{x}^r + \bm{y}\lambda) = \cos(x_1^r)(1 - \frac{y_1^2\lambda^2}{2}) + \cos(x_2^r)(1 - \frac{y_2^2\lambda^2}{2}) - \sin(x_1^r)y_1\lambda - \sin(x_2^r)y_2\lambda$.

Now suppose the $g(\bm{x}^r + \bm{y}\lambda)$ is the lower bound of $f(\bm{x}^r + \bm{y}\lambda)$, such that $f(\bm{x}^r + \bm{y}\lambda) \geq g(\bm{x}^r + \bm{y}\lambda) \geq 0$. The lower bound problem of \cref{probtrans} is defined as:
 \begin{equation}
      \begin{split}
          \min~& - \lambda\\
          \text{s.t. }
                & \forall y_i^2 \leq 1, i = 1,2,\ldots,n, \\
                &g(\bm{x}^r + \bm{y}\lambda)\geq 0.
      \end{split}
      \label{problowerbound}
\end{equation}
    
\subsection{Refute Certification Problem}\label{subsec:refute_prob}
We can show that the \textit{local positiveness constraint} in \eqref{problowerbound} is satisfied by showing its refute set is empty. The refute set is constructed as:
\begin{equation}
    \begin{split}
    \begin{cases}
      \gamma_0^* = -g(\bm{x}^r + \bm{y}\lambda) \geq 0\\
       \gamma_i^* = 1 - y_i^2 \geq 0, i = 1,2,\ldots,n 
    \end{cases}
        \label{eq:refute}
    \end{split}
\end{equation}
Then we can use \cref{theo:posi} to turn the emptiness problem for \eqref{eq:refute} to a feasibility problem similar to \eqref{eq:pos}. Finally, the problem \cref{problowerbound} can be turned into the following optimization problem:
\begin{equation}
    \begin{split}
        &\min~ -\lambda,\\
        \text{ s.t. } & \exists p_i\in\text{SOS},i = 0,1,2,\ldots, \text{ such that } p_0 + p_1\gamma_0^* + \ldots +   \\
         & p_s\gamma_n^* + p_{01}\gamma_0^*\gamma_1^* + \ldots + p_{012\ldots n}\gamma_0^*\ldots\gamma_n^* + 1 = 0.
    \end{split}
    \label{probafterposi}
\end{equation}
\eqref{probafterposi} is equivalent to \eqref{problowerbound}, which will be proved in \cref{lem:refute}.

Here we highlight that problem \eqref{probafterposi} inherently differentiates from the state-of-the-art SOSP due to the fact that set of inequalities in \eqref{eq:refute} contain an additional decision variable $\lambda$ to be optimized, which is not the free variable $\bm{x}$ in \cref{theo:posi}.

\subsection{Sufficient Nonlinear Programming Problem}\label{SDP}
To efficiently search for the existence of SOS polynomials $\{p_i\}$ while maximizing the desired tolerance bound $\lambda$, we set $p_i$ to be a positive scalar $\alpha_i$ for all $i\geq1$. By defining a decision vector $c^* = \begin{bmatrix} \lambda & \alpha_1 & \alpha_2 & \ldots & \alpha_{012\ldots n} \end{bmatrix}$ and corresponding weight vector $w^* = \begin{bmatrix} -1 & 0 & 0 & \ldots & 0\end{bmatrix}$.  A simplified problem of \eqref{probafterposi} can be defined:
\begin{equation}
    \begin{split}
        &\min~ {w^*}^T c^*,\\
        \text{ s.t. } &p_0 = -\alpha_1\gamma_0^* \ldots - \alpha_s\gamma_n^* - \alpha_{01}\gamma_0^*\gamma_1^* -\ldots\\
        &- \alpha_{012\ldots n}\gamma_0^*\ldots\gamma_n^* - 1  \in\text{SOS}, \\
        & \alpha_i \geq 0, i = 1,2,\ldots
    \end{split}
    \label{probsufficient}
\end{equation}
It will be shown in \cref{lem:suffi} that the solution for \eqref{probsufficient} is sufficient to satisfy the constraints \eqref{probafterposi}, hence provides a lower bound to $\lambda$ in \eqref{probafterposi}.

To solve for \eqref{probsufficient}, suppose the degree of $p_0$ is $2d$, we first do a sum-of-squares decomposition of $p_0$ such that $p_0 = Y^\top Q^*(\lambda, \alpha_1, \alpha_2, \ldots , \alpha_{012\ldots n}) Y$, where $Q^*$ is symmetric and $Y = \begin{bmatrix}1 & y_1 & y_2 & \ldots & y_n & y_1y_2 & \ldots & y_n^d \end{bmatrix}$. Specifically, for off-diagonal terms $Q^*_{ij}$ as the element of $Q^*$ at $i$-th row and $j$-th column ($i \neq j$), assuming the the coefficient of the term $Y_iY_j$ in $p_0$ is $w_{ij}$, we set $Q^*_{ij} = \frac{w_{ij}}{2}$. Note that multiple $(i,j)$ combinations may result in same $Y_iY_j$. We recommend to select one $(i,j)$ combination to be $\frac{w_{ij}}{2}$ while setting the rest to be zero. 

With the decomposed $Q^*$, we can formulate the equivalent nonlinear programming problem of \eqref{probsufficient} according to \cref{lem:sdp}: 
\begin{equation}
    \begin{split}
        &\min~ {w^*}^T c^*,\\
        \text{ s.t. } 
        & det(Q^*(\lambda, \alpha_1, \alpha_2,\ldots , \alpha_{012\ldots n})_k) \geq 0, k = 1,2,\ldots\\
        & \alpha_i \geq 0, i = 1,2,\ldots
    \end{split}
    \label{probfinal}
\end{equation}
where $Q^*(\lambda, \alpha_1, \alpha_2, \ldots , \alpha_{012\ldots s})_k$ denotes the $k \times k$ submatrix consisting of the first $k$ rows and columns of $Q^*(\lambda, \alpha_1, \alpha_2, \ldots , \alpha_{012\ldots n})$. Note that \eqref{probfinal} is a standard nonlinear programming problem. It can be solved efficiently by off-the-shelf nonlinear programming solvers, such as \verb+fmincon+ in MATLAB.

\begin{algorithm}
\caption{Joint Tolerance Estimation}\label{alg:sos}
\begin{algorithmic}[1]
\Procedure{JTE}{$\bm{x}^r$, $f(\cdot)$} 
\State Construct lower bound problem \eqref{problowerbound}.
\State Construct refute certification problem \eqref{probafterposi}.
\State Construct sufficient nonlinear programming problem \eqref{probfinal}
\State Solve \eqref{probfinal} using nonlinear programming solver.
\EndProcedure
\end{algorithmic}
\end{algorithm}

The integrated algorithm for joint tolerance estimation is summarized in \cref{alg:sos}, which efficiently finds the lower bound solution of \eqref{probori}.

\section{Properties of JTE}\label{sec:proof}
In this section, we will prove that the solution from \Cref{alg:sos} is the lower bound of the optimal solution of \eqref{probori}. The main result is summarized in the following theorem. 
\begin{theorem}[Feasibility of \Cref{alg:sos}]
Under \Cref{alg:sos}, the solution of problem \eqref{probfinal}  provides the lower bound of problem \eqref{probori}, \ie denoting the solution of problem \eqref{probfinal} as $\lambda_*$, the optimal solution for \eqref{probori} as $\lambda$, then $0 \leq \lambda^* \leq \lambda$.
\label{theo1}
\end{theorem} 

\subsection{Preliminary Results.}
Before proving \cref{theo1}, we present preliminary results that are useful toward proving the theorem. 
By change of variables, \eqref{probtrans} is equivalent to \eqref{probori} as summarized in \cref{lem:trans}. By using a conservative constraint $g$ instead of $f$, \eqref{problowerbound} provides a lower bound of \eqref{probtrans} as summarized in \cref{lem:lower}.
 \cref{lem:refute} shows problem with \textit{local positiveness constraints} can be transformed into a problem that searches for SOS polynomials, \ie \eqref{probafterposi}, by constructing the corresponding refute certifications. 
\cref{lem:suffi} shows that we can solve a sufficient problem of \eqref{probafterposi}, which only searches for one type of SOS polynomials as shown in \eqref{probsufficient}. Finally, \cref{lem:sdp} shows \eqref{probsufficient} is equivalent to the nonlinear program \eqref{probfinal}.

\begin{lemma}[Transformation]
Supposing the solution of \eqref{probori} is $\lambda_1$ and the solution of \eqref{probtrans} is $\lambda_2$. Then we have $\lambda_1 = \lambda_2$.
\label{lem:trans}
\end{lemma}

\begin{lemma}[Lower Bound]
Supposing the solution of \eqref{probtrans} is $\lambda_2$ and the solution of \eqref{problowerbound} is $\lambda_3$. If $g(\bm{x})\leq f(\bm{x})$ for all $\bm{x}$, then we have $\lambda_2 \geq \lambda_3$.
\label{lem:lower}
\end{lemma}

\begin{lemma}[Refute Equivalency]
If $\lambda_4$ is the solution of \eqref{probafterposi}, and $\lambda_3$ is the solution of \eqref{problowerbound}, then $\lambda_4 = \lambda_3$ .
\label{lem:refute}
\end{lemma}
\begin{proof}
Based on \cref{theo:posi} and the fact that we do not have inequation constraints ($\neq$) or equality constraints ($=$), the \textit{Positivstellensatz} conditions can be reduced to the following equivalence statement:
  \begin{gather}
      \textit{The set}
          \left\{
            x \in \mathbb{R}^n \bigg|\begin{array}{ll}
                                    \gamma_j(\bm{x})\geq 0,j = 1,..,s\\
                                    \end{array}
          \right\}
          \textit{ is empty.}\nonumber\\
        \Longleftrightarrow  \textit{There exists } \gamma \in \Gamma \textit{ such that } \gamma+1 = 0.
          \label{eq:pos1}
  \end{gather}
  For the constraint $\forall \| \bm{x} - \bm{x}^r\|_{\infty} \leq \lambda, f(\bm{x}) \leq 0$, its refute certification is $\forall \| \bm{x} - \bm{x}^r\|_{\infty}\leq \lambda, f(\bm{x}) > 0$. By change of variables, we refute certifications of the constraints in \eqref{problowerbound} can be written as  \eqref{eq:refute}. By applying \eqref{eq:pos1} and \cref{prop} for $\Gamma$, we know  \eqref{eq:refute} is equivalent to the following statement:
\begin{equation}\label{eq:equi_refute}
    \begin{split}
        &\exists p_i\in\text{SOS},i = 0,1,2,\ldots, \text{ such that } \gamma +1 = 0 \text{ and}\\
       & \gamma = p_0 + p_1\gamma_0^* + \ldots  + p_{01}\gamma_0^*\gamma_1^* + \ldots + p_{012\ldots n}\gamma_0^*\ldots\gamma_n^*.
    \end{split}
\end{equation}
  Therefore, problem \eqref{problowerbound} is equivalent to problem \eqref{probafterposi}.

\end{proof}

\begin{lemma}[Sufficiency]
Supposing $\lambda_5$ is the solution of \eqref{probsufficient}, and $\lambda_4$ is the solution of \eqref{probafterposi}, then $\lambda_5 \leq \lambda_4$.
\label{lem:suffi}
\end{lemma}
\begin{proof}
It has been shown that it is impossible to check every instance of the possible sum-of-squares polynomial~\cite{parrilo2003semidefinite}. Nonetheless, the simplest SOS is a positive constant scalar, \ie $\alpha \geq 0 \in \text{SOS}$. By substituting $p_i$ with $\alpha_i\geq0$ for $i = 1,2,\ldots$ and defining a new decision vector $c^* = \begin{bmatrix} \lambda & \alpha_1 & \alpha_2 & \ldots & \alpha_{012\ldots n} \end{bmatrix}$ and weight vector $w^* = \begin{bmatrix} -1 & 0 & 0 & \ldots & 0\end{bmatrix}$,  \eqref{probafterposi} can be rewrote as \eqref{probsufficient}. Denote the searching space for SOS polynomials $p_j, j = 1,2,\dots$ in \eqref{probafterposi} as $\Omega = [\omega_1, \omega_2, \ldots]$, where $\omega_i \in \mathbb{R}$ is the coefficient of $i$-th term of $p_j$. The search space for SOS polynomials in \eqref{probsufficient} is $\Omega^* = [\omega_1^*]$, where $\omega_1^* \in \mathbb{R}^{+}$ is the scalar term. Therefore, we have $\Omega^* \subseteq \Omega$, and $\lambda_5$ that solves for $\Omega^*$ also solves for $\Omega$, not vice versa. Since $0$ is the common solution for both $\Omega^*$ and $\Omega$, and their common solution sets evolve continuously, Thus we have $\lambda_5 \leq \lambda_4$.
\end{proof}

\begin{lemma}[Semi-Definite Cone Equivalency]
Supposing $\lambda_6$ is the solution of \eqref{probfinal}, and $\lambda_5$ is the solution of \eqref{probsufficient}
, then $\lambda_5 = \lambda_6$.
\label{lem:sdp}
\end{lemma}
\begin{proof}
Whether a polynomial is SOS can be efficiently verified through sum-of-squares decomposition~\cite{parrilo2003semidefinite}. 
Mathematically, given polynomial $h(\bm{x})$ whose degree is $2d$, the sum-of-squares decomposition decomposes $h(\bm{x})$ as:
\begin{equation}
    h(x) = z^T Q z, \quad z = \begin{bmatrix}1 & x_1  & \ldots & x_n & x_1x_2 & \ldots & x_n^d \end{bmatrix}
    \label{eq:gram}
\end{equation}
where $Q$ is a constant matrix. If $Q$ is positive semidefinite, then $h(\bm{x})\in\text{SOS}$. The condition that $Q\succeq0$ can be encoded in a semidefinite program~\cite{parrilo2003semidefinite}. In \eqref{probsufficient}, we directly enforce that the determinant of all submatrices of are positive. By linear algebra, we know that
\begin{equation}
    Q \succeq 0  \Longleftrightarrow det(Q_k) \geq 0 \text{ for all } k = 1,2,\ldots
    \label{eq:sdp}
\end{equation}
where $Q_k$ is the $k \times k$ submatrix consisting of the first $k$ rows and columns of $Q$ and $det(Q_k)$ denotes the determinant of submatrix $Q_k$. Therefore, the optimization problem \eqref{probsufficient} is equivalent to \eqref{probfinal}.
\end{proof}

\subsection{Proof of the Main Result}
\begin{proof}
  If $\lambda^*$ is the solution of \eqref{probfinal} and $\lambda_5$ is the solution of \eqref{probsufficient}. By \cref{lem:sdp}, we have $\lambda^* = \lambda_5$. If $\lambda_4$ is the solution of \eqref{probafterposi}, by \cref{lem:suffi}, we have $\lambda_5 \leq \lambda_4 $. If $\lambda_3$ is the solution of \eqref{problowerbound}, then by \cref{lem:refute}, we have $\lambda_3 = \lambda_4$. If $\lambda_2$ is the solution of \eqref{probtrans}, by \cref{lem:lower}, we have $\lambda_2 \geq \lambda_3$. Then, if $\lambda$ is the solution of \eqref{probori}. By \cref{lem:trans} we have $\lambda = \lambda_2$.
  From the above deduction, we have the main result $0\leq \lambda^* = \lambda_5 \leq \lambda_4 = \lambda_3 \leq \lambda_2 = \lambda$.
\end{proof}

\subsection{Additional Results}
\begin{theorem}[Scalability]\label{theo_multi}
If there exists multiple constraints $f_i(\bm{x}) \geq 0, i=1,2,\ldots,n$ and the solution for these constraints are $\lambda_i$, i.e., $\forall \|\bm{x}-\bm{x}^r\|_p\leq \lambda_i, f_i(\bm{x}) \geq 0, i=1,2,\ldots,n$.
Then $\lambda_{min} = \min\limits_i \lambda_i$ is a lower bound of the optimization $\min \lambda$ s.t. $\forall \|\bm{x}-\bm{x}^r\|_p\leq \lambda, \forall i, f_i(\bm{x})\geq 0$. In other words, $\forall \|\bm{x}-\bm{x}^r\|_p\leq \lambda_{min}, f_i(\bm{x}) \geq 0, i=1,2,\ldots,n$.
\end{theorem}
\begin{proof}
A straightforward proof is that, suppose $\lambda_i\leq\lambda_j, i \neq j$, we have $\{\bm{x}\mid \|\bm{x}-\bm{x}^r\|_p\leq \lambda_i\}\subseteq\{\bm{x}\mid \|\bm{x}-\bm{x}^r\|_p\leq \lambda_j\}$, then $\forall \|\bm{x}-\bm{x}^r\|_p\leq \lambda_i$, $f_i(\bm{x}) \geq 0$ and $f_j(\bm{x}) \geq 0$.
For the minimum $\lambda_{min}$, we have the intersection $\{\bm{x}\mid \|\bm{x}-\bm{x}^r\|_p\leq \lambda_{min}\} = \bigcap\limits_i \{\bm{x}\mid \|\bm{x}-\bm{x}^r\|_p\leq \lambda_{i}\}$. Hence, $\forall \|\bm{x}-\bm{x}^r\|_p\leq \lambda_{min}$, $f_i(\bm{x}) \geq 0$ for all $i=1,2,\ldots,n$.
\end{proof}
\Cref{theo_multi} implies that the computation complexity  grows linearly with the number of constraints if we solve each constraint independently and then minimize over all $\lambda_i$.

\begin{figure}
    \centering
    \includegraphics[width=1\columnwidth]{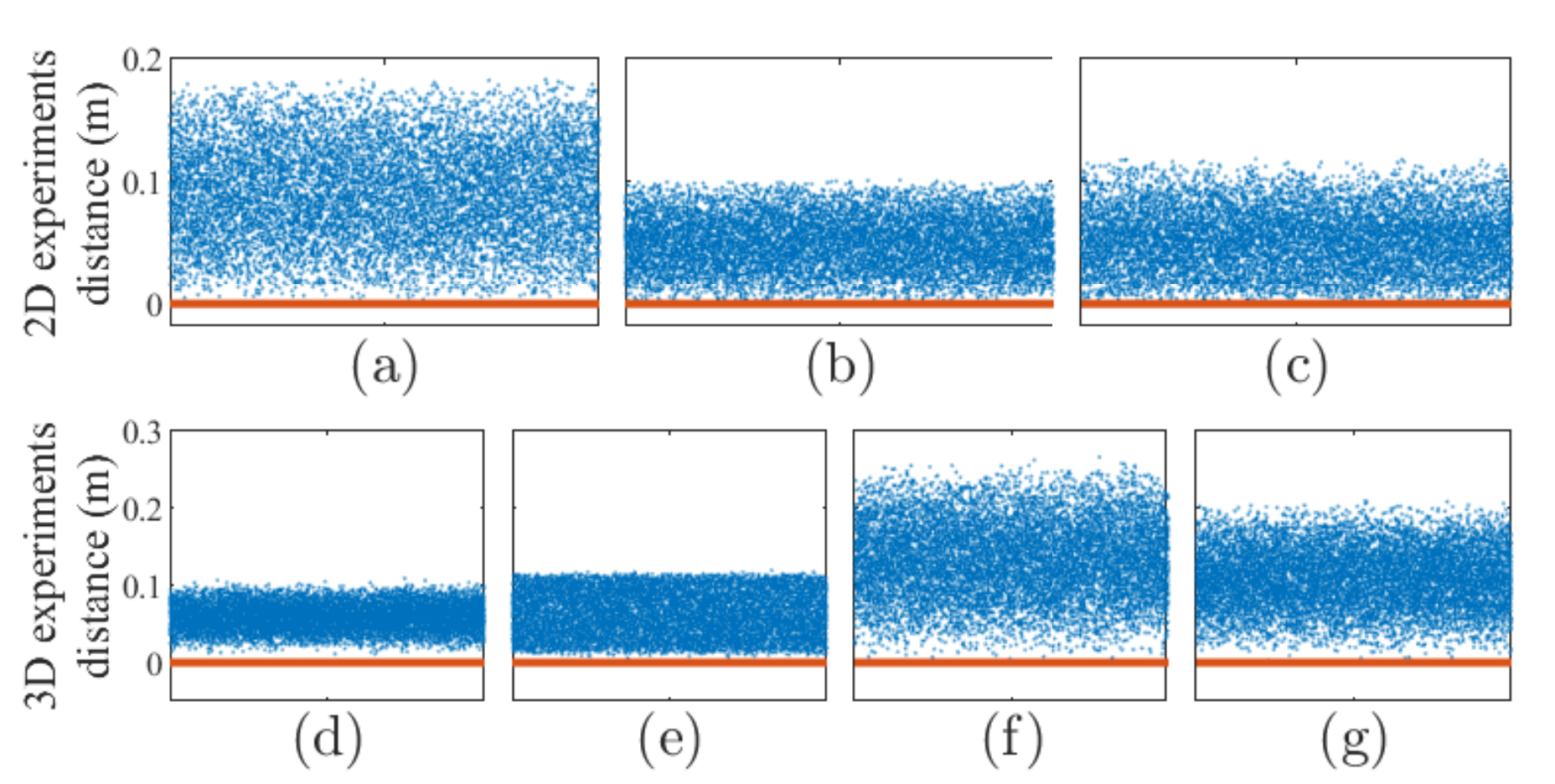}
    \caption{Validation of the algorithm for joint tolerance estimation on 2D problem and 3D problem. The blue points are sampled $\bm{x}$ in the hypercube $\|\bm{x}-\bm{x}^r\|\leq \lambda$ where $\lambda$ is the result returned by \eqref{alg:sos}. The horizontal axis is the sample ID. The vertical axis is the distance from the robot to the obstacle at that sample. These plots correspond to different constraints: (a) 2D - x axis half plane, $\xi_x = 1.456$. (b) 2D - y axis half plane, $\xi_y = 1.416$. (c) 2D - general half plane, $-x - y + 2.8 \geq 0$. (d) 3d - x axis half plane, $\xi_x^* = 1.8$. (e) 3D - y axis half plane, $\xi_y^* = 0.45$. (f) 3D - z axis half plane, $\xi_z^* = 1.35$. (g) 3D - general half plane, $-0.4758x -0.0135y - z + 1.7601 \geq 0$. The red line indicates the boundary of safety constraint.}
    \label{fig:comps}
\end{figure}

\section{Experiments and Results}\label{sec:result}
\subsection{Experimental Setup}\label{sec:exp_setup}
To demonstrate the effectiveness of the proposed JTE Toolbox for joint tolerance estimation of robot arms, we design two sets of experiments. In these experiments, the obstacle is always taken as a half plane. It is not an over simplification by taking half planes. As shown in the convex feasible set algorithm \cite{liu2017convex}~\cite{zhao2020contact}, all obstacles can be bounded by a set of half planes and we only need to stay away from these half planes. Moreover, by \cref{theo_multi}, we can solve the bound for each half plane independently and then take the minumum.

We design two sets of joint tolerance estimation problems: one is a 2DOF robot working in 2-dimensional space; the other is a 6DOF YASKAWA GP50 robot working in 3-dimensional Cartesian space. For all problem sets, we only enforce the constraint on the end effector since the other parts of the robot cannot reach the obstacle. The algorithm is generalizable to any other critical body point. Denote the end-effector position as $FK(x)$ where $FK(\cdot)$ computes forward kinematics. For each problem set, we randomly pick the reference robot configuration, and specify the following half plane constraints: 
\begin{itemize} 
\item For 2-dimensional plane 
     \begin{itemize} 
        \item half plane on x axis, such that $Fk(x)_{x} \leq \xi_x$; 
        \item half plane on y axis, such that $Fk(x)_{y} \leq \xi_y$;
        \item general half plane $O$, such that $d(FK(x),O) \geq 0$.
     \end{itemize}
\item For 3-dimensional Cartesian space
    \begin{itemize} 
        \item half plane on x axis, such that $Fk(x)_{x} \leq \xi_x^*$; 
        \item half plane on y axis, such that $Fk(x)_{y} \leq \xi_y^*$;
        \item half plane on z axis, such that $Fk(x)_{z} \leq \xi_z^*$;
        \item general half plane $O^*$, such that $d(FK(x),O^*) \geq 0$.
     \end{itemize}
\end{itemize}
The parameters $\xi_x, \xi_y, \xi_x^*, \xi_y^*, \xi_z^*$ are scalars, $FK(x)_x$, $FK(x)_y$, $FK(x)_z$ denote the $x,y,z$ coordinates of the end effector, and $d(\cdot)$ denotes a signed distance function to compute the distance between a point and a half plane. The experiment platforms for both settings are illustrated in \cref{platform}.

\begin{figure}
     \centering
    \begin{subfigure}[t]{0.22\textwidth}
        \raisebox{-\height}{\includegraphics[width=\textwidth]{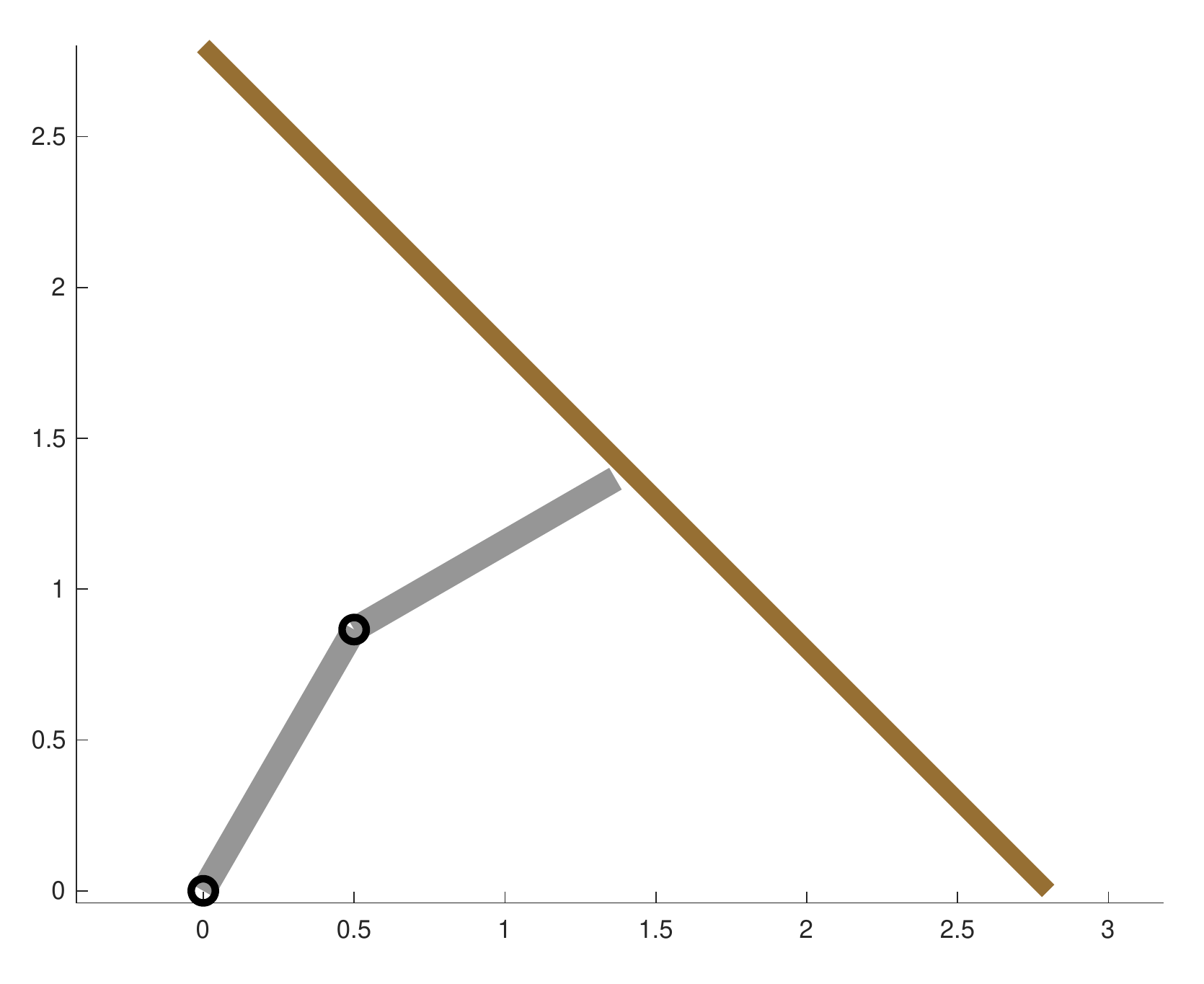}}
        \caption{2-dimensional plane space experiment platform}
    \end{subfigure}
    \begin{subfigure}[t]{0.209\textwidth}
        \raisebox{-\height}{\includegraphics[width=\textwidth]{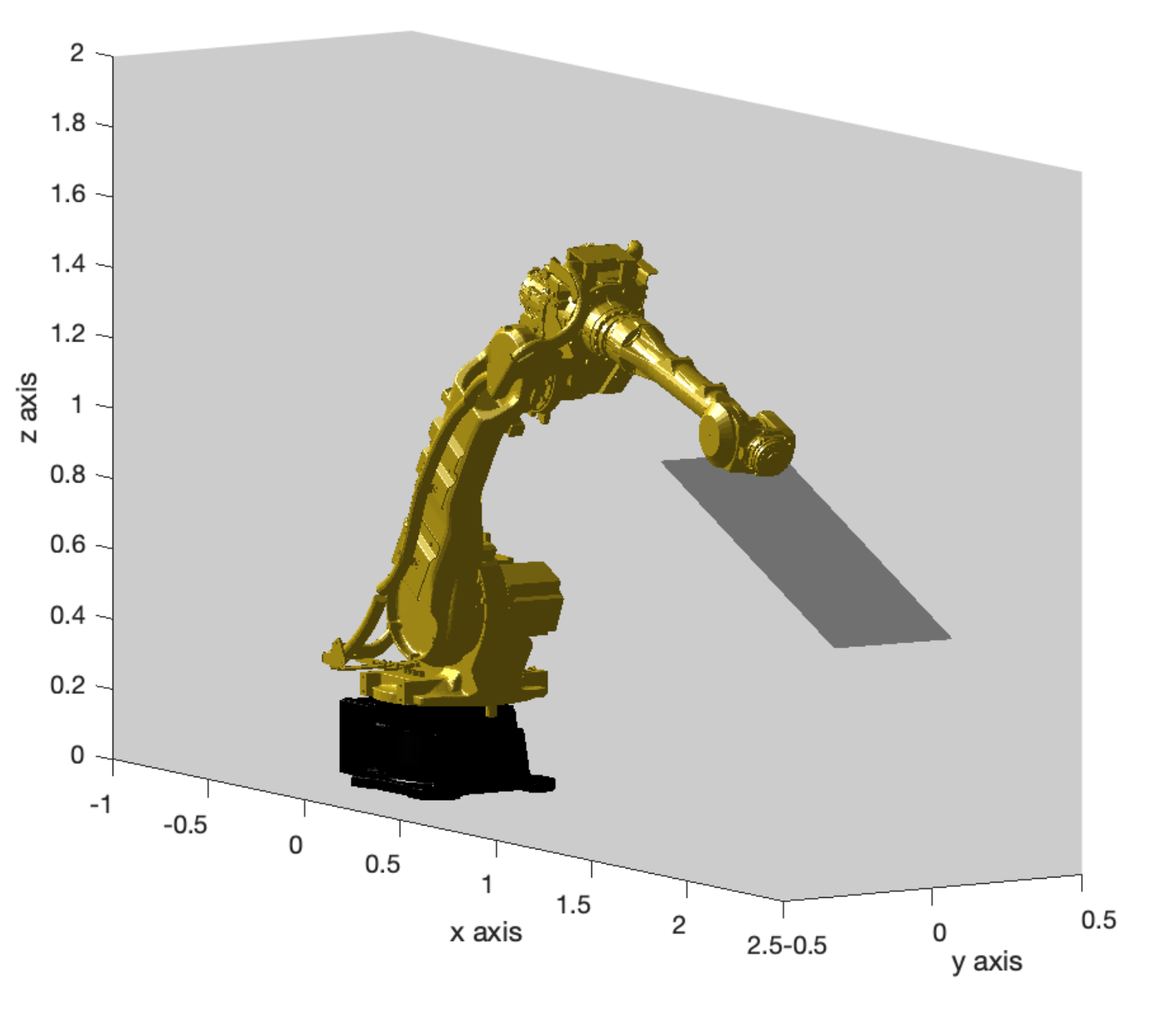}}
        \caption{3-dimensional Cartesian space experiment platform}
    \end{subfigure}
     \caption{Experimental platforms.}
     \label{platform}
\end{figure}

 \begin{table*}[htbp]
 \caption{ JTE Toolbox safe joint tolerance bound optimization performance summary cross two sets of experiments in terms of computation time, the safety constraints violation percentage, and the closest distance (safe distance) from sampled end-effector position to safety boundary.}
\centering
\begin{tabular}{@{}p{0.14\textwidth}*{9}{L{\dimexpr0.12\textwidth-2\tabcolsep\relax}}@{}}
\toprule
& \multicolumn{3}{c}{2-dimensional experiments} &
\multicolumn{4}{c}{3-dimensional experiments}\\
\cmidrule(r{4pt}){2-4} \cmidrule(l){5-8} 
& x-axis plane & y-axis plane & general plane &  x-axis plane & y-axis plane & z-axis plane & general plane\\
\midrule
time (\si{\second}) & 0.5953 & 0.8105 & 0.5940 & 6.804 & 5.028 & 1.5844 & 3.9198 \\
safety violation & 0$\%$ & 0$\%$ & 0$\%$ & 0$\%$ & 0$\%$ & 0$\%$ & 0$\%$ \\
smallest $f(\bm{x})$ (\si{\meter})  & 0.0035 & 0.0012 & 0.0021 & 0.0111 & 0.0067 & 0.0014 & 0.0043 \\
$\lambda$ (\si{\radian})  & 0.0670 & 0.0372 & 0.1145 & 0.0346 & 0.0265  & 0.035 & 0.0302 \\
\bottomrule
\end{tabular}
 \label{summarize}
\end{table*}

\subsection{Results}

For the 2D problem, the lengths of both links of the robot are  $1$m. The reference configuration is $\bm{x}^r = [\frac{\pi}{3},\frac{\pi}{6}]$. The parameters in the safety constraints are 1)  $\xi_x = 1.4560$m, 2)  $\xi_y = 1.4160 m$, and 3) the general half plane $O$ represented by $-x - y + 2.8 \geq 0$. 

For the 3D problem, the physical properties of the robot  follow the standard YASKAWA GP50 definition\footnote{https://www.motoman.com/en-us/products/robots/industrial/assembly-handling/gp-series/gp50}. The reference configuration is $\bm{x}^r = [\frac{\pi}{20},-\frac{\pi}{2},\frac{\pi}{20},\frac{\pi}{20},\frac{\pi}{20},\frac{\pi}{20}]$.  The parameters in the safety constraints are 1)  $\xi_x^* = 1.8 m$, 2) $\xi_y^* = 0.45 m$, 3) $\xi_z^* = 1.35 m$, 4) the general half plane $O^*$ represented by $-0.4758x -0.0135y - z + 1.7601 \geq 0$. All the experiments are performed on the MATLAB 2020 platform with a $2.3 GHz$ Intel Core i5 Processor.

To evaluate the optimality of the computed joint tolerance $\lambda$, we randomly sample $10000$ configurations $\bm{x}$ inside the hypercube $\|\bm{x}-\bm{x}^r\|_\infty\leq\lambda$, then check their distance to the obstacle, \eg $f(\bm{x})$. The samples for the 2D problem and the 3D problem are plotted in \cref{fig:comps}. We can observe that the computed joint tolerance bound is maximized in the sense that the minimum distance from sampled points in the tolerance bound to the obstacle is very close to $0$. Nonetheless, there is no safety violation across all the experiments, which aligns with \cref{theo1}. The computation time, percentage of safety violation, the smallest distance value $f(\bm{x})$ are reported in \Cref{summarize}, as well as the optimized tolerance bound $\lambda$. Note that smaller positive safe distance means smaller gap between the computed $\lambda$ and the ground truth $\lambda$, thus indicating better solution optimality. Statistical results indicate that JTE can find maximum joint tolerance bound with great optimality in real time where the safe distance is bounded at millimeter scale. At the same time, the great joint tolerance bound optimality does not compromise safety where zero safety violation can be observed across all sets of experiments. Note that $\lambda$ of all experiments are less than $0.244$ \si{\radian}, which satisfy the requirement of small angle approximation.


\section{CONCLUSIONS}

This paper presented a sum-of-squares programming method that can efficiently solve for the joint tolerance of robot arms with respect to safety constraints. The proposed method is proved to provide a tight lower bound of the joint tolerance. Numerous numerical studies are conducted and the results demonstrate the effectiveness of our proposed method in terms of computational efficiency and optimality. 
In the future, we will investigate more complex situations with non differentiable safety constraints.

\bibliographystyle{IEEEtran.bst}
\bibliography{reference}

\end{document}